\documentclass{article}
\usepackage[utf8]{inputenc}
\usepackage[english]{babel}
\usepackage[top=1in,bottom=1in,left=1in,right=1in]{geometry}
\usepackage[round]{natbib}

\usepackage{amsmath, amsfonts, amssymb, bbm}
\usepackage{amsthm}
\usepackage{mathtools}
\usepackage{graphicx}
\usepackage{mathrsfs}
\usepackage{bm}
\usepackage{tikz}

\newcommand{\R}{\mathbb{R}}
\newcommand{\N}{\mathbb{N}}

\newcommand{\wc}{\mathcal{C}}
\newcommand{\sly}{\bullet}
\newcommand{\sgn}{\mathrm{sgn}}

\newcommand{\sym}{\mathfrak{S}}

\newcommand{\rao}{\odot}
\newcommand{\tvec}{\mathrm{vec}}
\newcommand{\diag}{\mathrm{diag}}
\newcommand{\one}{\mathbbm 1}
\newcommand{\rank}{\mathrm{rank}}
\newcommand{\var}{\mathbb V}
\newcommand{\compl}{\mathsf c}
\newcommand{\dom}{\mathrm{dom}}
\newcommand{\img}{\mathrm{im}}
\newcommand{\intr}{\mathrm{int}}
\newcommand{\indA}{\mathcal A}
\newcommand{\D}{\mathcal{D}}

\DeclareRobustCommand{\loongrightarrow}{%
  \DOTSB\relbar\joinrel\relbar\joinrel\rightarrow
}

\DeclareRobustCommand{\looongrightarrow}{%
  \DOTSB\relbar\joinrel\relbar\joinrel\relbar\joinrel\relbar\joinrel\rightarrow
}

\DeclareRobustCommand{\loooongrightarrow}{%
  \DOTSB\relbar\joinrel\relbar\joinrel\relbar\joinrel\relbar\joinrel\relbar\joinrel\relbar\joinrel\relbar\joinrel\relbar\joinrel\rightarrow
}

\newtheorem{theorem}{Theorem}[section]
\newtheorem{lemma}[theorem]{Lemma}

\begin{document}

\title{Interpolation with deep neural networks with non-polynomial activations: necessary and sufficient numbers of neurons\thanks{This work was partially funded by a UBC DSI Postdoctoral Fellowship, NSERC Discovery Grant No. 2021-03677, and NSERC ALLRP 581098-22.}}
\author{Liam Madden\thanks{Department of Electrical and Computer Engineering, University of British Columbia, Vancouver, BC, Canada.}}
\maketitle

\begin{abstract}
The minimal number of neurons required for a feedforward neural network to interpolate $n$ generic input-output pairs from $\R^d\times \R^{d'}$ is $\Theta(\sqrt{nd'})$. While previous results have shown that $\Theta(\sqrt{nd'})$ neurons are sufficient, they have been limited to sigmoid, Heaviside, and rectified linear unit (ReLU) as the activation function. Using a different approach, we prove that $\Theta(\sqrt{nd'})$ neurons are sufficient as long as the activation function is real analytic at a point and not a polynomial there. Thus, the only practical activation functions that our result does not apply to are piecewise polynomials. Importantly, this means that activation functions can be freely chosen in a problem-dependent manner without loss of interpolation power.
\end{abstract}

\section{Introduction}

Neural networks were first conceived by Warren McCulloch and Walter Pitts in 1943 as a computational model inspired by neurons in the brain \citep{mcculloch1943logical}. Fifteen years later, Frank Rosenblatt developed the first perceptron, a two-layer neural network with Heaviside activation \citep{rosenblatt1958perceptron}. Today, neural networks are the building block for many machine learning models. In particular, they are one of the key ingredients in the modern transformer model, which has found great success in the realm of natural language processing \citep{vaswani2017attention}. While the original inspiration for neural networks came from biology, it is not clear that their success is at all related to the analogy with brains. In fact, the memory capacity perspective instead sees them as no more than simple mappings that are, nevertheless, expressive enough to interpolate data sets.

The memory capacity of a machine learning model is the largest $n$ such that it can interpolate $n$ generic input-output pairs~\citep{cover1965geometrical}, where by generic we mean that the set of exceptions lies on the zero set of a nontrivial real analytic function and therefore is measure zero and closed~\citep[Corollary 10]{gunning1965analytic}. We will consider the setting where inputs come from $\R^d$ and outputs come from $\R^{d'}$. As an example, a two-layer feedforward neural network (FNN) is a mapping $h\circ g\circ f:\R^d\to\R^{d'}$ where $f:\R^d\to\R^m$ is linear, $g:\R^m\to\R^m$ is an element-wise mapping, and $h:\R^m\to\R^{d'}$ is linear. While a linear mapping $\R^d\to\R^{d'}$ cannot interpolate generic data sets, it turns out that $h\circ g\circ f$ can~\citep{baum1988multilayer,yun2019small,bubeck2020network,madden2024memory}.

More generally, an $L$-layer FNN is a mapping $f_L\circ g_{L-1}\circ f_{L-1}\cdots g_1\circ f_1$ where $f_{\ell}:\R^{m_{\ell-1}}\to\R^{m_{\ell}}$ is linear for all $\ell\in[L]$ and $g_{\ell}:\R^{m_{\ell}}\to\R^{m_{\ell}}$ is an element-wise mapping for all $\ell\in[L]$. The element-wise mappings are called the activation functions and $\sum_{\ell=1}^{L}m_{\ell}$ is called the number of neurons. If we allow the linear mappings to be tuned, then there are $\sum_{\ell=1}^Lm_{\ell}(m_{\ell-1}+1)$ tunable parameters. If the element-wise mappings are continuously differentiable, and if the number of parameters is less than $nm_L$, then, for all $x_1,\ldots,x_n\in\R^{m_0}$, the set of $y_1,\ldots,y_n\in\R^{m_L}$ for which the data set can be interpolated is measure zero by Sard's theorem~\citep{sard1941measure}. In other words, the memory capacity is upper bounded by the number of parameters divided by $m_L$. Proportional lower bounds have been proved for FNNs with sigmoid, Heaviside, and ReLU activations~\citep{sakurai1992networks,yamasaki1993lower,huang2003learning,vershynin2020memory}, but not for general activations. We prove a proportional lower bound for three-layer FNNs only assuming the activation is real analytic at a point and not a polynomial there. This includes common activation functions such as tanh, arctan, and GELU. In fact, the only practical activation functions which are excluded are piecewise polynomials. We also extend to $L$-layer FNNs by using the first $L-3$ layers as preparation and the final three for interpolation.
But, the importance of depth is already evident for three-layer FNNs.

Let $L\ge 3$. We show, in Theorem~\ref{thm:necessary}, that $\sqrt{2nd'}+\Omega(1)$ neurons are necessary for an $L$-layer FNN to interpolate $n$ generic data points. Then, in Theorem~\ref{thm:sufficient}, we show that $2\sqrt{2nd'}+\Omega(1)$ neurons are sufficient for an $L$-layer FNN to interpolate $n$ generic data points. Thus, the necessary and sufficient conditions we show are within a factor of two of each other.

\subsection{Results}

First, in Theorem~\ref{thm:necessary}, we rigorously prove a condition on the number of neurons necessary to interpolate $n$ generic data points. While it is well known that $\Omega(\sqrt{nd'})$ neurons are necessary, we prove a more precise condition.

Then, in Theorem~\ref{thm:threelayer}, we lower bound the memory capacity of a three-layer FNN with activations which are real analytic at a point and not a polynomial there. To do so, we first, in Theorem~\ref{thm:threerank}, lower bound the generic rank of the Jacobian of a three-layer FNN with respect to its middle layer. To do that, we first lower bound the generic rank of $\phi(\psi(uv^\top)wz^\top)\sly\psi(uv^\top)$ where $\sly$ is the face-splitting product. We do this for polynomial $\phi$ and $\psi$ in Theorem~\ref{thm:threepoly}, then extend to real analytic $\phi$ and $\psi$ in Theorem~\ref{thm:threereal}.

Next, let $L\ge 4$. In Theorem~\ref{thm:deep}, we lower bound the memory capacity of an $L$-layer FNN with activations which are real analytic at a point and, for the first $L-2$ activations, nontrivial there; for the last two activations, non-polynomial there. To do so, we first, in Theorem~\ref{thm:narrowrankfull}, lower bound the generic rank of the Jacobian of a four-layer FNN with one neuron in its first layer, with respect to its third layer. We are able to reduce the deep FNN to this specific FNN using Lemma~\ref{lma:reduce}.

Finally, with the memory capacity lower bounds of Theorems~\ref{thm:threelayer} and~\ref{thm:deep} in hand, we are able to prove, in Theorem~\ref{thm:sufficient}, a condition on the number of neurons sufficient to interpolate $n$ generic data points. The necessary and sufficient conditions are asymptotically, with respect to $n$, equal up to a factor of two.

\subsection{Related work}

The memory capacity---also known as finite sample expressivity, memorization capacity, storage capacity, or, simply, capacity---of a machine learning model with $k$ parameters is the largest $n$ such that it can interpolate $n$ generic input-output pairs, where by generic we mean that the set of exceptions lies on the zero set of a nontrivial real analytic function and therefore is measure zero and closed~\citep[Corollary 10]{gunning1965analytic}. The idea of memory capacity goes back to \cite{cover1965geometrical} who considered the separating capacities of families of surfaces. Later, \cite{baum1988multilayer} proved that a two-layer FNN with Heaviside activation has memory capacity at least $\approx k$ where outputs are in $\{\pm 1\}$.
\cite{sakurai1992networks} extended this to three-layer FNNs. \cite{huang1991bounds} proved it for two-layer FNNs with sigmoid activation and outputs in $\R$. \cite{yamasaki1993lower} sketched a proof for $L$-layer FNNs with sigmoid activation. \cite{huang2003learning} proved it for three-layer FNNs with sigmoid activation. \cite{yun2019small} proved it for two-layer FNNs with ReLU activation and outputs in $\{\pm 1\}$ (their Corollary 4.2), and for three-layer FNNs with ReLU activation and outputs in $\R$ (their Theorem 3.1). \cite{bubeck2020network} proved it for two-layer FNNs with ReLU activation and outputs in $\R$. \cite{madden2024memory} proved it for two-layer FNNs with general activations (excluding only low degree polynomials and low degree splines) and outputs in $\R$.

There is also a line of recent works which make assumptions on the separability of the input data in order to prove memory capacity results. Specifically, given $n,d\in\N$ and $\delta>0$, define $\D(n,d,\delta)=\{x_1,\ldots,x_n\in\R^d\mid \delta\max_{i\neq j}\|x_i-x_j\|<\min_{i\neq j}\|x_i-x_j\|\}$. \cite{vershynin2020memory}, \cite{rajput2021exponential}, and \cite{park2021provable} proved memory capacity results for generic data sets with the input set coming from $\D(n,d,\delta)$.
\cite{vershynin2020memory} proved that an $L$-layer FNN with Heaviside or ReLU activation has memory capacity at least $\approx k-\exp(\delta^{-2})$ where outputs are in $\{0, 1\}$.
\cite{rajput2021exponential} proved that an $L$-layer FNN with Heaviside activation has memory capacity at least $\approx k-d\delta^{-1}$ where outputs are in $\{0, 1\}$. \cite{park2021provable} proved that a variable-depth FNN and sigmoid or ReLU activation can approximate, up to arbitrary precision, any data set of size at most $\approx (k-\log\delta^{-1})^{3/2}$ where outputs are in $\{1,2\}$, and any data set of size of at most $\approx k-\log\delta^{-1}-\log C$ where outputs are in $\{1,\ldots,C\}$ for some $C\in\N$.
It is easy to see that the interior of the complement of $\D(n,d,\delta)$ is nonempty, so these results do not extend to generic data sets with inputs coming from $\R^d$.

There is also a similar line of research studying the minimum singular value of the Jacobian of the mapping, given input data, from parameters to output data. This is useful from an optimization perspective because gradient descent converges at a linear rate when the minimum singular value is large enough. Moreover, when the minimum singular value is positive and there are more parameters than data points, i.e. when the Jacobian has rank $n$, the mapping is surjective. In the context of $L$-layer FNNs, \cite{bombari2022memorization} showed that the minimum singular value is positive with high probability over the data set when: (1) the activation function is non-linear, Lipschitz continuous, and has Lipschitz continuous gradient; (2) the width of subsequent layers increases by no more than a constant multiplicative constant; and (3) the final hidden layer has asymptotically more than $n\log^8(n)$ parameters. Thus, to get memorization with only $O(\sqrt{n})$ neurons, it is necessary that $L=\Omega(\log(\sqrt{n}/d))$, where $d$ is the dimension of the feature vectors. In other words, their result does not given the optimal number of neurons when $L=3$. Furthermore, the number of parameters is only optimal up to log factors and the result only holds with high probability over data sets, rather than for generic data sets. \cite{bombari2022memorization} built off of the work of \cite{nguyen2021tight}, removing the requirement in \cite{nguyen2021tight} that one of the widths be on the order of $n\log^2(n)$.

\subsection{Organization}

In Section~\ref{sec:prelims}, we go through the necessary preliminaries. In Section~\ref{sec:necessary}, we present the full FNN model and prove the necessary condition on the number of neurons. In Section~\ref{sec:three}, we prove the lower bound on the memory capacity of a three-layer FNN. In Section~\ref{sec:deep}, we prove the lower bound on the memory capacity of a deep FNN. In Section~\ref{sec:sufficient} we prove that these memory capacity lower bounds lead to a sufficient condition on the number of neurons that is, asymptotically, only twice the necessary condition.

\section{Preliminaries}
\label{sec:prelims}

Throughout the paper we use the following notation: $a\lor b$ denotes $\max\{a,b\}$, $a\land b$ denotes $\min\{a,b\}$, $[n]$ denotes $\{1,\ldots,n\}$, $\binom{A}{n}$ denotes $\{B\subset A\mid |B|=n\}$, $\tvec$ denotes the column-wise vectorize operation, $e_k\in\R^n$ denotes the $k$th coordinate vector, $\one_n$ denotes the vector of ones in $\R^n$, $a^{(k)}$ indicates that the exponent $k$ is applied to the vector $a$ element-wise, $\sym_n$ denotes the symmetric group of degree $n$, $\rao$ denotes the Khatri-Rao product (the column-wise Kronecker product), and $\sly$ denotes the face-splitting product (the row-wise Kronecker product). Given two sequences $(a_n)$ and $(b_n)$, we write $a_n=o(b_n)$ if $\lim_{n\to\infty} |a_n/b_n|=0$, $a_n=\Omega(b_n)$ if
$\limsup |b_n/a_n|<\infty$, and $a_n=\Theta(b_n)$ if $a_n=\Omega(b_n)$ and $b_n=\Omega(a_n)$. Given a matrix $A$ we use $a_k$ to denote its $k$th column. Given vectors $(a_k)_{k=1}^n$ we use $[a_k]_{k=1}^n$ to denote the matrix $[a_1|\cdots|a_n]$. Let $\ell\in\N$ and $K\subset \N\cup\{0\}$. Let $r\in \ell K$. Then $\wc(r,\ell,K)$ denotes the set of compositions of $r$ into $\ell$ parts in $K$ \citep{heubach2004compositions}. If $\indA\subset \R^d$, then we order it lexicographically. Moreover, if $\{a_1,\ldots,a_n\}\in\binom{\indA}{n}$ (where $a_1<\cdots <a_n$), then we identify it with the matrix $[a_1|\cdots|a_n]^T$, and so write $[a_1|\cdots|a_n]^T\in\binom{\indA}{n}$.

Let $M$ be a manifold and let $f:M\to\R^n$. We use $\var(f)$ to denote $\{x\in M\mid f(x)=0\}$.
If $f$ is nontrivial and real analytic, then, by Corollary 10 of \cite{gunning1965analytic}, $\var(f)^\compl$ is measure zero and closed. Generally, it is quite easy to see that a particular $f$ is real analytic, the harder part is showing that it is nontrivial. But notice how useful it is to characterize a set in this way: if there is a single point $x\in M$ such that $f(x)\neq 0$, then $\var(f)^\compl$ is measure zero and closed. This leads us to define generic, similarly to \cite{allman2009identifiability}, to mean that the set of exceptions lies on the zero set of a nontrivial, real analytic function. Note that, in addition to being measure zero and closed, the zero set of a nontrivial, real analytic function is locally a finite union of lower-dimensional manifolds~\citep{guaraldo1986topics}.

We use Sard's theorem~\citep[Thm. 6.10]{lee2013smooth} and the Constant Rank Theorem~\citep[Thm.~4.12]{lee2013smooth} from differential topology. The latter underlies Lemma~\ref{lma:everything}, which we have borrowed from \cite{madden2024memory}. We also use
the Cauchy-Binet formula \citep[Sec. I.2.4]{gantmacher1960matrices} and the Leibniz determinant formula \citep[Def. 10.33]{axler2015linear} from linear algebra.

\section{The FNN model}
\label{sec:necessary}

Let $d,d'\in\N$ and suppose data comes from $\R^d\times\R^{d'}$. Then a FNN parameterized by $\theta$ is a mapping $h_\theta:\R^d\to\R^{d'}$ defined in the following way. Let $L\in\N$. This is the number of hidden layers. The general case for $L=1$ was already dealt with in \cite{madden2024memory}, so we will assume $L\ge 2$. Let $\psi_\ell:\R\to\R$ for all $\ell\in[L]$. These are the activation functions. Let $m_1,\ldots,m_L\in\N$. These are the widths of each layer respectively.
Let $W_\ell\in\R^{m_{\ell-1}\times m_\ell}~\forall\ell\in [L]$. These are the hidden layer weight matrices.
Let $b_\ell\in\R^{m_\ell}~\forall \ell\in[L]$. These are the bias vectors.
Let $V\in\R^{m_L\times d'}$. This is the output layer weight matrix.
Then the FNN with parameters $ (W_1,b_1,\ldots,W_L,b_L,V)$ is the following composition of mappings:
\begin{align}
\label{eq:model}
    \underset{\R^d}{\rule[-.3cm]{.1pt}{1cm}}\overset{\psi_1\left(W_1^\top \cdot~+b_1\right)}{\loooongrightarrow}\underset{\R^{m_1}}{\rule[-1.3cm]{.1pt}{3cm}}\overset{\psi_2\left(W_2^\top \cdot~+b_2\right)}{\loooongrightarrow}\underset{\R^{m_2}}{\rule[-1.3cm]{.1pt}{3cm}}\cdots \overset{\psi_L\left(W_L^\top \cdot~+b_L\right)}{\loooongrightarrow}\underset{\R^{m_L}}{\rule[-1.3cm]{.1pt}{3cm}}\overset{V^\top\cdot}{\loongrightarrow}\underset{\R^{d'}}{\rule[-.3cm]{.1pt}{1cm}}.
\end{align}
We will denote it by $h_\theta$, where $\theta\coloneqq (W_1,b_1,\ldots,W_L,b_L,V)$, and call it an $(L+1)$-layer FNN with activations $(\psi_\ell)$, widths $(m_\ell)$, and parameters $\theta$.
Note that it has $\sum_{\ell=1}^{L-1} m_\ell m_{\ell+1}+dm_1+\one_L^\top m+d'm_L$ parameters total and $\one_L^\top m+d'$ neurons total. We have the following condition on the number of neurons necessary to interpolate $n$ generic points in $\R^d\times\R^{d'}$.

\begin{theorem}
\label{thm:necessary}
Let $n,d,d',L\in\N$ with $L\ge 2$. Then an $(L+1)$-layer FNN with continuously differentiable activations and less than
\begin{align*}
    \sqrt{2nd'+(d\lor d'+1)^2-2d\land d'-4L+5}-d\lor d'+d'+L-2
\end{align*}
neurons cannot interpolate $n$ generic points in $\R^d\times\R^{d'}$.
\end{theorem}
\begin{proof}
Let $\{h_\theta\mid\theta\}$ be an $(L+1)$-layer FNN with undetermined parameters and continuously differentiable activations as defined in Eq.~\eqref{eq:model}. Let $x_1,\ldots,x_d\in\R^d$ and define $F:\theta\mapsto [h_\theta(x_i)]_{i=1}^n$. If the total number of parameters is less than $nd'$, then, by Sard's theorem, the image of $F$ has measure zero. Thus, if the total number of parameters is less than $nd'$, then $\{h_\theta\mid\theta\}$ cannot interpolate $n$ generic points in $\R^d\times\R^{d'}$. The total number of parameters in $\{h_\theta\mid\theta\}$ is $\sum_{\ell=1}^{L-1} m_\ell m_{\ell+1}+dm_1+\one_L^\top m+d'm_L$ and the total number of neurons is $\one_L^\top m+d'$. To turn the necessary condition on the number of parameters into a necessary condition on the number of neurons, we will lower bound the optimization problem
\begin{align*}
    q_\N\coloneqq &\min_{m\in\N^L}\quad \one_L^\top m+d'\\
    &~~\text{s.t.}~~\quad \sum_{\ell=1}^{L-1} m_\ell m_{\ell+1}+dm_1+\one_L^\top m+d'm_L\ge nd'.
\end{align*}
Define
\begin{align*}
    q_\R=&\min_{m\in\R^L}\quad \one_L^\top m+d'\\
    &~~\text{s.t.}~~\quad \sum_{\ell=1}^{L-1} m_\ell m_{\ell+1}+dm_1+\one_L^\top m+d'm_L\ge nd'.
\end{align*}
For each $L,b\in\N$ such that $2\le L\le b$, define
\begin{align*}
    p(b)=&\max_{m\in\R^L}\quad \sum_{\ell=1}^{L-1} m_\ell m_{\ell+1}+dm_1+\one_L^\top m+d'm_L\\
    &~~\text{s.t.}~~\quad \one_L\preceq m,\quad\one_L^\top m\le b.
\end{align*}
Then we have
\begin{align*}
    q_\N \ge  q_\R = \min\{b\ge L\mid p(b)\ge nd'\}+d'.
\end{align*}
By Young's inequality,
\begin{align*}
    p(b)\le&\max_{m\in \R^L}\quad \frac{m_1^2}{2}+\sum_{\ell=2}^{L-1} m_\ell^2+\frac{m_L^2}{2}+(d+1)m_1+\sum_{\ell=2}^{L-1}m_\ell+(d'+1)m_L\\
    &~~\text{s.t.}~~\quad \one_L\preceq m,\quad\one_L^\top m\le b.
\end{align*}
The right-hand side is a maximization problem of a convex function over a nonempty, compact, convex set, so, by Corollary 32.3.1 of \cite{Rockafellar1970}, the maximum is attained at an extreme point of the set. The extreme points of the set are $\one_L$ and $\one_L+(b-L)e_\ell$ for each $\ell\in[L]$. If $d\ge d'$, then $\one_L+(b-L)e_1$ is a maximizer. If $d\le d'$, then $\one_L+(b-L)e_L$ is a maximizer. So,
\begin{align*}
    p(b)\le \frac{1}{2}(b-L+1)^2+(d\lor d'+1)(b-L+1)+d\land d'+2L-\frac{5}{2}.
\end{align*}
Thus,
\begin{align*}
    q_\N\ge \sqrt{2nd'+(d\lor d'+1)^2-2d\land d'-4L+5}-d\lor d'+d'+L-2,
\end{align*}
proving the theorem.
\end{proof}

To get a sufficient condition on the number of neurons, we will restrict to the case $d'=1$ and extend to more general $d'$ afterwards. For all $X\in\R^{d\times n}$, define
\begin{align*}
    F_{(m_\ell),n}(X,W_1,\ldots,W_L,v)=v^\top\psi_L\left(W_L^\top\cdots \psi_1\left(W_1^\top X\right)\cdots\right)\in\R^n.
\end{align*}
We will often denote $F_{(m_\ell),n}$ by $F$ with $(m_\ell)$ and $n$ clear from the dimensions of the inputs. We include bias vectors in the full model but only need $F$ in the proofs.

Given $X\in\R^{d\times n}$ and $y\in\R^n$, the following lemma gives a sufficient condition for the equation $y^\top=F(X,W_1,\ldots,W_L,v)$ to have a solution.

\begin{lemma}[Thm. 5.2 of \cite{madden2024memory}]
\label{lma:everything}
Let $n,d,m\in\N$.
Let $M\subset \R^d$ be open. Let $f:M\to\R^{n\times m}$ be $C^1$. Define $\Tilde{f}:M\times M\to\R^{n\times 2m}:(w,u)\mapsto [f(w)~f(u)]$. For all $v,z\in\R^m$, define $F_v:M\to\R^n:w\mapsto f(w)v$ and $\Tilde{F}_{v,z}:M\times M\to\R^n:(w,u)\mapsto \Tilde{f}(w,u)[v;z]$. If there exists $v_0\in\R^m$ and $w_0\in M$ such that $\rank(DF_{v_0}(w_0))=n$, then $\Tilde{F}$ is surjective as a function of $(v,z)\in\R^m\times\R^m$ and $(w,u)\in M\times M$.
\end{lemma}

One consequence of Lemma~\ref{lma:everything} is that, for every $X\in\R^{d\times n}$, if there is a single $(W_1,\ldots,W_L,v)$ such that the Jacobian of $F_{(m_\ell),n}$ with respect to the final hidden layer has rank $n$, then it follows that $y^\top =F_{(m_1,\ldots,m_{L-1},2m_L),n}(X,W_1,\ldots,W_L,v)$ has a solution for all $y\in\R^n$. In fact, we will show that the Jacobian with respect to the final hidden layer has rank $n$ for \textit{generic} $(X,W_1,\ldots,W_L,\one_{m_L})$ as long as $m_L(m_{L-1}-1)\ge n$.
The Jacobian with respect to the final hidden layer is
\begin{align*}
    \partial_{\tvec(W_L)}F(X,W_1,\ldots,W_L,v) = G(X,W_1,\ldots,W_L,v)^\top&\\
    \text{where }G(X,W_1,\ldots,W_L,v)\coloneqq \diag(v)\psi_L'\left(W_L^\top \hat{X}\right)\rao \hat{X}&\\
    \text{with }\hat{X} \coloneqq \psi_{L-1}\left(W_{L-1}^\top\cdots\psi_1\left(W_1^\top X\right)\cdots\right).&
\end{align*}

\section{Three layers}
\label{sec:three}

First, we will consider the case $L=2$. Here, the FNN with parameters $(W,b,U,c,v)$ is the following composition of mappings:
\begin{align*}
    \underset{\R^d}{\rule[-.3cm]{.1pt}{1cm}}\overset{\psi\left(W^\top \cdot~+b\right)}{\loooongrightarrow}\underset{\R^{m}}{\rule[-1.3cm]{.1pt}{3cm}}\overset{\phi\left(U^\top \cdot~+c\right)}{\loooongrightarrow}\underset{\R^{\ell}}{\rule[-1.3cm]{.1pt}{3cm}}\overset{v^\top\cdot}{\loongrightarrow}\underset{\R}{\bm{\cdot}}
\end{align*}
With biases set to zero, the Jacobian with respect to the second layer is
\begin{align*}
    \partial_{\tvec(U)}F(X,W,U,v)=\phi'\left(\psi\left(X^\top W\right)U\right)\diag(v)\sly \psi\left(X^\top W\right).
\end{align*}
So, by Lemma~\ref{lma:everything}, we can get a memory capacity result by lower bounding the generic rank of $\phi'(\psi(X^\top W)U)\sly \psi(X^\top W)$. The rank result is Theorem~\ref{thm:threerank} and the memory capacity result is Theorem~\ref{thm:threelayer}.

We prove Theorem~\ref{thm:threerank} by first lower bounding the generic rank of $\phi(\psi(uv^\top)wz^\top)\sly\psi(uv^\top)$ in Theorem~\ref{thm:threereal}. To see that this is sufficient, let $I\subset [n]$ and $J\subset [m\ell]$ such that $|I|=|J|$. Let $f(u,v,w,z)=\det_{I,J}(\phi(\psi(uv^\top)wz^\top)\sly\psi(uv^\top))$ and $g(X,W,U)=\det_{I,J}(\phi'(\psi(X^\top W)U)\sly \psi(X^\top W))$. If $f\not\equiv 0$, then there exists $(u,v,w,z)$ such that $g(\one_du^\top/\sqrt{d},\one_dv^\top/\sqrt{d},wz^\top)=f(u,v,w,z)\neq 0$, so $g\not\equiv 0$. Thus, Theorem~\ref{thm:threereal} implies Theorem~\ref{thm:threerank}.

To prove Theorem~\ref{thm:threereal}, we first prove it when $\phi$ and $\psi$ are polynomials of sufficiently high degree---Theorem~\ref{thm:threepoly}---then extend to non-polynomial real analytic functions using Taylor's theorem. The proof of Theorem~\ref{thm:threepoly} is the most difficult proof in the paper, so we will sketch it here.

First, we decompose $\phi(\psi(uv^\top)wz^\top)\sly\psi(uv^\top)$ as a linear combination of rank-one matrices. Then, we apply the Cauchy-Binet formula to get
\begin{align*}
    \det_{I,J}\left(\phi\left(\psi\left(uv^\top\right)wz^\top\right)\sly\psi\left(uv^\top\right)\right)=\sum_{k,\ell}\left(\sum_{r}\xi_{k,\ell,r}p_{k,\ell,r}(a)\right)q_{k,\ell}(b,c)=p(a,b,c)
\end{align*}
where $p$, the $p_{k,\ell,r}$, and the $q_{k,\ell}$ are polynomials. We want to show that $p\not\equiv 0$. We will do this in three steps: (1) construct $(k^*,\ell^*)$ such that $q_{k^*,\ell^*}$ is linearly independent from the other $q_{k,\ell}$, (2) construct $r^*$ such that $p_{k^*,\ell^*,r^*}$ is linearly independent from the other $p_{k^*,\ell^*,r}$, and (3) show that $\xi_{k^*,\ell^*,r^*}\neq 0$. Both of the first two steps will require induction arguments.

\begin{theorem}
\label{thm:threepoly}
Let $n,d,m\in\N$. Let $K\subset \N\cup\{0\}$ and $L\subset \N\cup\{0\}$ such that $\min\{|K|,|L|\}\ge \lfloor n/(d-1)\rfloor (d-1)$. Let $\alpha_k\in\R\backslash\{0\}~\forall k\in K$ and $\beta_{\ell}\in\R\backslash\{0\}~\forall \ell\in L$. Define $\psi(x)=\sum_{k\in K}\alpha_k x^k$ and $\phi(x)=\sum_{\ell\in L}\beta_{\ell}x^{\ell}$. Then there exists a nontrivial polynomial function $f:\R^n\times \R^d\times\R^d\times\R^m\to\R$ such that, for all $(u,v,w,z)\in\var(f)^\compl$,
\begin{align*}
    \rank\left(\psi\left(uv^\top \right)\sly \phi\left(\psi\left(uv^\top \right)wz^\top \right)\right)\ge \min\{m,\lfloor n/(d-1)\rfloor\}(d-1).
\end{align*}
\end{theorem}
\begin{proof}
Let $\Tilde{n},\Tilde{d},\Tilde{m}\in\N$. Define $d=\widetilde{d}-1$, $n=d\lfloor \widetilde{n}/d\rfloor$, and $m=n/d$. Define $I=[n]$ and $J=[d]\times [m]$. Let $u\in\R^{\Tilde{n}}$, $v,w\in\R^{\Tilde{d}}$, and $z\in\R^{\Tilde{m}}$. We want to show that $\det_{I,J}(\psi(uv^\top )\sly \phi(\psi(uv^\top )wz^\top ))$ is nonzero for generic $(u,v,w,z)$. To do so, we just need to construct a single example such that this is the case. Towards this end, set $v_{\Tilde{d}}=1$ and $w=e_{\widetilde{d}}$. Then, observe,
\begin{align*}
    \psi\left(uv^\top \right)=\sum_{k\in K}\alpha_k\left(uv^\top \right)^{(k)}=\sum_{k\in K}\alpha_k u^{(k)}v^{(k)T}
\end{align*}
and
\begin{align*}
    \phi\left(\psi\left(uv^\top \right)wz^\top \right)&=\sum_{\ell\in L}\beta_{\ell}\left(\psi\left(uv^\top \right)w\right)^{(\ell)}z^{(\ell)T}\\
    &=\sum_{\ell\in L}\beta_{\ell}\left(\sum_{k\in K}\alpha_k u^{(k)}\right)^{(\ell)}z^{(\ell)T}\\
    &=\sum_{\ell\in L,k\in K^{\ell}}\beta_{\ell}\alpha_{k_1}\cdots\alpha_{k_{\ell}} u^{(k_1+\cdots+k_{\ell})}z^{(\ell)T}\\
    &=\sum_{\ell\in L,r\in \ell K}\beta_{\ell}\underset{\coloneqq \gamma_{\ell,r}}{\underbrace{\sum_{k\in\wc(r,\ell,K)}\alpha_{k_1}\cdots\alpha_{k_{\ell}}}} u^{(r)}z^{(\ell)T}
\end{align*}
where we use the convention that $\alpha_{k_1}\cdots\alpha_{k_{\ell}}=1$, $k_1+\cdots+k_{\ell}=0$, and $\gamma_{\ell,r}=1$ if $\ell=0$. Next, using that $ab^\top \sly cy^\top =(a\circ c)(b\otimes y)^\top $,
\begin{align*}
    \psi\left(uv^\top \right)\sly \phi\left(\psi\left(uv^\top \right)wz^\top \right)=\sum_{k\in K,\ell \in L,r\in \ell K}\alpha_{k}\beta_{\ell}\gamma_{\ell,r} u^{(k+r)}\left(v^{(k)}\otimes z^{(\ell)}\right)^\top .
\end{align*}
Let $\indA$ denote the set of indices. Let $a$ denote the vector of the first $n$ entries of $u$, $b$ the first $d$ entries of $v$, and $c$ the first $m$ entries of $z$. Then, applying the Cauchy-Binet formula,
\begin{align*}
    p(a,b,c)&\coloneqq \det_{I,J}\left(\psi\left(uv^\top \right)\sly \phi\left(\psi\left(uv^\top \right)wz^\top \right)\right)\\
    &= \sum_{[k|\ell|r]\in\binom{\indA}{n}}\left(\prod_{i=1}^n\alpha_{k_i}\beta_{\ell_i}\gamma_{\ell_i,r_i}\right)\det\left(\left[a^{(k_i+r_i)}\right]_{i=1}^n\right)\det\left(\left[b^{(k_i)}\otimes c^{(\ell_i)}\right]_{i=1}^n\right)\\
    &\coloneqq\sum_{[k|\ell|r]\in\binom{\indA}{n}}\xi_{k,\ell,r}p_{k,\ell,r}(a)q_{k,\ell}(b,c).
\end{align*}

We want to show that $p$ is not identically zero. To start, if $[k|\ell]\in\binom{K\times L}{s}$ for some $s<n$, then $q_{k,\ell}\equiv 0$ since there will be repeat columns. Thus, we can restrict to $[k|\ell]\in \binom{K\times L}{n}$ and $r_i\in \ell_i K~\forall i\in [n]$. From this, we get
\begin{align*}
    p(a,b,c) = \sum_{[k|\ell]\in\binom{K\times L}{n}}\underset{\coloneqq p_{k,\ell}(a)}{\underbrace{\left(\sum_{r_i\in \ell_i K~\forall i\in [n]}\xi_{k,\ell,r}p_{k,\ell,r}(a)\right)}}q_{k,\ell}(b,c).
\end{align*}

Note that, since $\min\{|K|,|L|\}\ge n$, there exists $k\in\binom{K}{n}$ and $\ell\in\binom{L}{n}$. We will complete the proof of the theorem with the following three steps. First, we will construct $k^*\in\binom{K}{n}$ and $\ell^*\in\binom{L}{n}$ such that $q_{k^*,\ell^*}$ is linearly independent from $q_{k,\ell}$ for all other $[k|\ell]\in\binom{K\times L}{n}$. Second, we will construct $r_i^*\in \ell_i^* K~\forall i\in[n]$ such that $p_{k^*,\ell^*,r^*}$ is linearly independent from $p_{k^*,\ell^*,r}$ for all other $r_i\in\ell_i^* K~\forall i\in[n]$. Third, we will show that $\xi_{k^*,\ell^*,r^*}\neq 0$. Then it follows that $p$ is not identically zero.

To begin step one, let $[k|\ell]\in\binom{K\times L}{n}$. Then, applying the Leibniz determinant formula, we get
\begin{align*}
    q_{k,\ell}&=\sum_{\sigma\in\sym(n)}\sgn(\sigma)\prod_{i=1}^d\prod_{j=1}^mb_i^{k_{\sigma(m(i-1)+j)}}c_{j}^{\ell_{\sigma(m(i-1)+j)}}\\
    &=\sum_{\sigma\in\sym(n)}\sgn(\sigma)\underset{\coloneqq q_{k,\ell,\sigma}}{\underbrace{\left(\prod_{i=1}^db_i^{\sum_{j=1}^m k_{\sigma(m(i-1)+j)}}\right)\left(\prod_{j=1}^mc_{j}^{\sum_{i=1}^d\ell_{\sigma(m(i-1)+j)}}\right)}}.
\end{align*}

Let $k^*$ be the smallest $n$ integers in $K$ and let $\ell^*$ be the smallest $n$ integers in $L$. Let $\sigma\in\sym(n)$. Let $\tau\in\sym(n)$ be the identity permutation. Suppose $q_{k^*,\ell^*,\sigma}=q_{k^*,\ell^*,\tau}$. Then $\sum_{j=1}^m k_{\sigma(m(i-1)+j)}^*=\sum_{j=1}^m k_{m(i-1)+j}^*~\forall i\in[d]$ and $\sum_{i=1}^d\ell_{\sigma(m(i-1)+j)}^*=\sum_{i=1}^d\ell_{m(i-1)+j}^*~\forall j\in[m]$. Thus, $\sigma=\tau$ since both $k_i^*$ and $\ell_{j}^*$ are increasing. Thus, the monomial $q_{k^*,\ell^*,\tau}$ has coefficient 1 in $q_{k^*,\ell^*}$.

Now, suppose $q_{k,\ell,\sigma}=q_{k^*,\ell^*,\tau}$. Then
$\sum_{j=1}^m k_{\sigma(m(i-1)+j)}=\sum_{j=1}^m k_{m(i-1)+j}^*~\forall i\in[d]$ and $\sum_{i=1}^d\ell_{\sigma(m(i-1)+j)}=\sum_{i=1}^d\ell_{m(i-1)+j}^*~\forall j\in[m]$. We will prove that $\sigma=\tau$, $k=k^*$, and $\ell=\ell^*$ with two induction steps.

First, $\sum_{j=1}^mk_{j}^*$ is the sum of the smallest $m$ integers in $K$. Thus, $\sigma([m])=[m]$ and $k_j=k_j^*~\forall j\in[m]$. Let $i\in[d-1]$. Suppose $\sigma([ms]\backslash[m(s-1)])=[ms]\backslash[m(s-1)]~\forall s\in[i]$ and $k_j=k_j^*~\forall j\in[mi]$. Then $\sum_{j=1}^mk_{mi+j}^*$ is the sum of the next smallest $m$ integers in $K$. Thus, $\sigma([m(i+1)]\backslash[mi])=[m(i+1)]\backslash[mi]$ and $k_j=k_j^*~\forall j\in[m(i+1)]$. So, by induction, $k=k^*$ and $\sigma([mi]\backslash[m(i-1)])=[mi]\backslash[m(i-1)]~\forall i\in[d]$.

We can prove with a similar induction step that $\ell=\ell^*$ and $\sigma(m[d]-m+j)=m[d]-m+j~\forall j\in[m]$. Putting the two properties of $\sigma$ together, we get that $\sigma=\tau$. Thus, the monomial $q_{k^*,\ell^*,\tau}$, which has coefficient 1 in $q_{k^*,\ell^*}$, has coefficient 0 in all other $q_{k,\ell}$. In other words, $q_{k^*,\ell^*}$ is linearly independent from $q_{k,\ell}$ for all other $[k|\ell]\in\binom{K\times L}{n}$, completing step one.

Moving on to step two, let $r_i\in \ell_i^* K~\forall i\in [n]$. Note that $p_{k^*,\ell^*,r}\equiv 0$ unless the $k_i^*+r_i$ are distinct, so suppose that this is the case. Then, applying the Leibniz determinant formula, we get
\begin{align*}
    p_{k^*,\ell^*,r} = \sum_{\sigma\in\sym(n)}\sgn(\sigma)\underset{\coloneqq p_{k^*,\ell^*,r,\sigma}}{\underbrace{a_1^{k_{\sigma(1)}^*+r_{\sigma(1)}}\cdots a_n^{k_{\sigma(n)}^*+r_{\sigma(n)}}}}.
\end{align*}

For each $i\in[n]$, let $r_i^*$ be the smallest integer in $\ell_i^* K$. Note that the $k_i^*$ are increasing and the $r_i^*$ are nondecreasing so the $k_i^*+r_i^*$ are increasing and therefore distinct. Let $\sigma\in\sym(n)$. Let $\tau\in\sym(n)$ be the identity permutation. Suppose $p_{k^*,\ell^*,r^*,\sigma}=p_{k^*,\ell^*,r^*,\tau}$. Then $k_{\sigma(i)}^*+r_{\sigma(i)}^*=k_i^*+r_i^*~\forall i\in[n]$. Thus, $\sigma=\tau$ since the $k_i^*+r_i^*$ are distinct. So, the monomial $p_{k^*,\ell^*,r^*,\tau}$ has coefficient 1 in $p_{k^*,\ell^*,r^*}$.

Now, suppose $p_{k^*,\ell^*,r,\sigma}=p_{k^*,\ell^*,r^*,\tau}$. Then $k_{\sigma(i)}^*+r_{\sigma(i)}=k_{i}^*+r_{i}^*~\forall i\in[n]$. We will prove that $\sigma=\tau$ and $r=r^*$ by induction on $i$.

First, $k_1^*+r_1^*$ is the sum of the smallest integer in $\{k_1^*,\ldots,k_n^*\}$ and the smallest integer in $\ell_1^* K\cup\cdots\cup\ell_n^* K$. Thus, $k_{\sigma(1)}^*=k_1^*$ and $r_{\sigma(1)}=r_1^*$; in other words, $\sigma(1)=1$ and $r_1=r_1^*$. Now, suppose $\sigma(i)=i$ and $r_i=r_i^*$ for all $i< s\le n$. Then, $k_s^*+r_s^*$ is the sum of the smallest integer in $\{k_s^*,\ldots,k_n^*\}$ and the smallest integer in $\ell_s^* K\cup\cdots\cup\ell_n^* K$. Thus, $k_{\sigma(s)}^*=k_s^*$ and $r_{\sigma(s)}=r_s^*$; in other words, $\sigma(s)=s$ and $r_s=r_s^*$. So, by induction, $\sigma=\tau$ and $r=r^*$.

So, the monomial $p_{k^*,\ell^*,r^*,\tau}$, which has coefficient 1 in $p_{k^*,\ell^*,r^*}$, has coefficient 0 in all other $p_{k^*,\ell^*,r}$. In other words, $p_{k^*,\ell^*,r^*}$ is linearly independent from $p_{k^*,\ell^*,r}$ for all other $r_i\in \ell_i^* K~\forall i\in[n]$, completing step two.

Moving on to step three, $\xi_{k^*,\ell^*,r^*}\neq 0$ if and only if $\gamma_{\ell_i^*,r_i^*}\neq 0~\forall i\in[n]$. Let $i\in[n]$. If $\ell_i^*=0$, then $\gamma_{\ell_i^*,r_i^*}=1\neq 0$. Suppose $\ell_i^*\neq 0$. Then, since $r_i^*$ is the smallest integer in $\ell_i^* K$, $\wc(r_i^*,\ell_i^*,K)$ has only one element, namely $(k_1^*,\ldots,k_1^*)$. Thus,
\begin{align*}
    \gamma_{\ell_i^*,r_i^*} = \alpha_{k_1^*}^{\ell_i^*}\neq 0,
\end{align*}
completing step three, and so completing the proof.
\end{proof}

\begin{theorem}
\label{thm:threereal}
Let $n,d,m\in\N$. Let $\psi:\R\to\R$ and $\phi:\R\to\R$ both be real analytic at zero and not a polynomial there. Let their radii of convergence at zero be $\rho$ and $\rho'$ respectively and define
\begin{align*}
    M=\{(u,v,w,z)\in\R^n\times\R^d\times\R^d\times\R^m\mid |u_iv_j|<\rho,|\psi(u_iv^\top )wz_k|<\rho'~\forall (i,j,k)\}.
\end{align*}
Then $0\in M$, $M$ is open, and there exists a nontrivial real analytic function $f:M\to\R$ such that, for all $(u,v,w,z)\in \var(f)^\compl$,
\begin{align*}
    \rank\left(\psi\left(uv^\top \right)\sly \phi\left(\psi\left(uv^\top \right)wz^\top \right)\right)\ge \min\{m,\lfloor n/(d-1)\rfloor\}(d-1).
\end{align*}
\end{theorem}
\begin{proof}
First, to show that $M$ is open, let $M'$ be the preimage of $(-\rho,\rho)^{n\times d}$ under $(u,v)\mapsto uv^\top $. Then $M$ can be seen as the preimage of $(-\rho,\rho)^{n\times d}\times(-\rho',\rho')^{n\times m}$ under $M'\times\R^d\times\R^m\to\R^{n\times d}\times\R^{n\times m}:(u,v,w,z)\mapsto (uv^\top ,\psi(uv^\top )wz^\top )$. The mapping is continuous, therefore $M$ is open.

Next, let $(\alpha_k)$ and $(\beta_k)$ be the coefficients of the Taylor expansions at zero of $\psi$ and $\phi$ respectively. Given $K\in\N$, define $\psi_K=\sum_{k=0}^K\alpha_kx^k$ and $\phi_K=\sum_{k=0}^K\beta_kx^k$. Let $I=[\lfloor n/(d-1)\rfloor(d-1)]$ and $J=[d-1]\times [\lfloor n/(d-1)\rfloor]$. Define
\begin{align*}
    f:M\to\R:(u,v,w,z)\mapsto \det_{I,J}\left(\psi\left(uv^\top \right)\sly \phi\left(\psi\left(uv^\top \right)wz^\top \right)\right).
\end{align*}
Let $K,L\in\N$ and define
\begin{align*}
    g_{K,L}:M\to\R:(u,v,w,z)\mapsto \det_{I,J}\left(\psi_K\left(uv^\top \right)\sly \phi_L\left(\psi_K\left(uv^\top \right)wz^\top \right)\right).
\end{align*}
If $K$ and $L$ are sufficiently large for $\psi_K$ and $\phi_L$ to both have at least $\lfloor n/(d-1)\rfloor(d-1)$ monomials, then the monomial $p_{k^*,\ell^*,r^*,\tau}q_{k^*,\ell^*,\tau}$ from the proof of Theorem~\ref{thm:threepoly} has coefficient $\gamma_{k^*,\ell^*,r^*}\neq 0$ in $g_{K,L}$. Moreover, $k^*$, $\ell^*$, and $r^*$ do not change as $K$ and $L$ increase further. Thus, the monomial $p_{k^*,\ell^*,r^*,\tau}q_{k^*,\ell^*,\tau}$ has coefficient $\gamma_{k^*,\ell^*,r^*}\neq 0$ in the Taylor expansion of $f$ at zero as well. In other words, the Taylor series of $f$ at zero has at least one nonzero coefficient, and so $f$ is not identically zero, proving the theorem.
\end{proof}

Theorem~\ref{thm:threereal} easily extends to general matrices which are not necessarily rank-one.

\begin{theorem}
\label{thm:threerank}
Let $n,d,m,\ell\in\N$. Let $\psi:\R\to\R$ and $\phi:\R\to\R$ both be real analytic at zero and not a polynomial there. Let their radii of convergence at zero be $\rho$ and $\rho'$ respectively and define
\begin{align*}
    M=\{(X,W,U)\in\R^{d\times n}\times\R^{d\times m}\times\R^{m\times \ell}\mid |x_i^\top w_j|<\rho,|\psi(x_i^\top W)u_k|<\rho'~\forall (i,j,k)\}.
\end{align*}
Then $0\in M$, $M$ is open, and there exists a nontrivial real analytic function $f:M\to\R$ such that, for all $(X,W,U)\in\var(f)^\compl$,
\begin{align*}
    \rank\left(\psi\left(X^\top W\right)\sly \phi\left(\psi\left(X^\top W\right)U\right)\right)\ge \min\{\ell,\lfloor n/(m-1)\rfloor\}(m-1)
\end{align*}
\end{theorem}
\begin{proof}
Let $f:M\to \R$ be the sum of squares of minors of order $\min\{\ell,\lfloor n/(m-1)\rfloor\}(m-1)$. To see that $f$ is nontrivial, let $(u,v,w,z)\in \var(g)^\compl$, where $g$ is the nontrivial real analytic function from Theorem~\ref{thm:threereal}, and set $X=\one_du^\top/\sqrt{d}$, $W=\one_dv^\top/\sqrt{d}$, and $U=wz^\top$.
\end{proof}

Now, we will apply Lemma~\ref{lma:everything} and Theorem~\ref{thm:threerank} to prove the following result, which includes bias vectors.

\begin{theorem}
\label{thm:threelayer}
Let $n,d,m,\ell\in\N$ such that $\ell\ge 2\lceil n/(m-1)\rceil$. Let $\psi:\R\to\R$ and $\phi:\R\to\R$ each be real analytic at a point and not a polynomial there. Then there exists a nontrivial real analytic function $f:\R^{d\times n}\backslash\{0\}\to\R$ such that, for all $X\in\var(f)^\compl$ and $y\in\R^n$, there exists $W\in\R^{d\times m}$, $b\in\R^m$, $U\in\R^{m\times \ell}$, $c\in\R^\ell$, and $v\in\R^\ell$ such that
\begin{align*}
    y = v^\top\phi\left(U^\top\psi\left(W^\top X+b\one_n^\top\right)+c\one_n^\top\right).
\end{align*}
\end{theorem}
\begin{proof}
Since the only requirement on $n,d,m,\ell$ is that they satisfy $\ell\ge2\lceil n/(m-1)\rceil$, we can assume, without loss of generality, that $(m-1)|n$.
Set $\ell'=\lfloor \ell/2\rfloor$. Let $\eta\in\R$ be a point where $\psi$ is real analytic and not a polynomial. Let $\zeta$ be such a point for $\phi$. By setting $b=\eta\one_m$ and $c=\zeta\one_\ell$, we can assume, without loss of generality, that $\eta=\zeta=0$ and remove the bias vectors. Set $v'=\one_{\ell'}$. Applying Theorem~\ref{thm:threerank}, there exists a nontrivial real analytic function $g:M\to\R$ such that, for all $(X',W',U')\in\var(g)^\compl$, $\rank(G_{(m,\ell'),n}(X',W',U',v'))=n$. Let $(X',W',U')\in\var(g)^\compl$. Using $\rho$ and $\rho'$ from the definition of $M$ in Theorem~\ref{thm:threerank}, define $I=(-\rho,\rho)\cap(-1,1)$, $a=\sup_{x\in\bar{I}}|\psi(x)|$, and
\begin{align*}
    f:\R^{d\times n}\backslash\{0\}\to\R:X\mapsto g\left(X,\frac{\max\{1,\rho\}W'}{2\|X\|_F\|W'\|_F},\frac{\rho' U'}{2a\|U'\|_{1,\infty}}\right).
\end{align*}
Then, for all $(i,j,k)$,
\begin{align*}
    \frac{\max\{1,\rho\}|x_i^\top w_j'|}{2\|X\|_F\|W'\|_F}\le \frac{\max\{1,\rho\}\|x_i\|_2\|w_j'\|_2}{2\|X\|_F\|W'\|_F}<\max\{1,\rho\},
\end{align*}
so
\begin{align*}
    \bigg|\psi\left(\frac{\max\{1,\rho\}x_i^\top W'}{2\|X\|_F\|W'\|_F}\right)\frac{\rho' u_k'}{2a\|U'\|_{1,\infty}}\bigg|\le \frac{a\rho' \|u_k'\|_1}{2a\|U'\|_{1,\infty}}< \rho',
\end{align*}
and so $f$ is well defined. Moreover, $f$ is nontrivial and real analytic. Let $X\in\var(f)^\compl$. Then $F_{(m,\ell),n}(X,\cdot)$ is surjective by Lemma~\ref{lma:everything}, completing the proof.
\end{proof}

\section{Four or more layers}
\label{sec:deep}

First, we will consider a four layer FNN with its first layer width equal to one. Here, the FNN with parameters $(u,z,W,z)$ is the following composition of mappings:
\begin{align*}
    \underset{\R^d}{\rule[-.3cm]{.1pt}{1cm}}\overset{\varphi\left(u^\top\cdot\right)}{\looongrightarrow}\underset{\R}{\bm{\cdot}}\overset{\psi\left(z \cdot\right)}{\looongrightarrow}\underset{\R^{m}}{\rule[-1.3cm]{.1pt}{3cm}} \overset{\phi\left(W^\top \cdot\right)}{\looongrightarrow}\underset{\R^{\ell}}{\rule[-1.3cm]{.1pt}{3cm}}\overset{v^\top\cdot}{\loongrightarrow}\underset{\R}{\bm{\cdot}}
\end{align*}

Essentially, since we are only solving for the final hidden layer anyway, we compress the data in the initial layers and only use the final three. First, we lower bound the rank of the Jacobian when the final hidden layer matrix is rank-one.

\begin{theorem}
\label{thm:narrowrank}
Let $n,d,m,\ell\in\N$.
Let $\varphi:\R\to\R$, $\psi:\R\to\R$, and $\phi:\R\to\R$ be real analytic at zero with radii of convergence $\rho$, $\rho'$, and $\rho''$ respectively. Assume $\varphi$ is nontrivial at zero. Assume $\psi$ and $\phi$ are not polynomials at zero.
Define
\begin{align*}
    M=\{(X,u,v,w,z)\in\R^{d\times n}\times\R^d\times\R^m\times\R^m\times\R^\ell\mid |&x_i^\top u|<\rho,\quad |\varphi(x_i^\top u)v_j|<\rho',\\
    &|\psi(\varphi(x_i^\top u)v^\top)wz_k|<\rho''~\forall (i,j,k)\}.
\end{align*}
Then $0\in M$, $M$ is open, and there exists a nontrivial real analytic function $f:M\to\R$ such that, for all $(X,u,v,w,z)\in\var(f)^\compl$,
\begin{align*}
    \rank\left(\psi\left(\varphi\left(X^\top u\right)v^\top \right)\sly \phi\left(\psi\left(\varphi\left(X^\top u\right)v^\top \right)wz^\top \right)\right)\ge \min\{\ell,\lfloor n/(m-1)\rfloor\}(m-1).
\end{align*}
\end{theorem}
\begin{proof}
Let $f$ be the nontrivial real analytic function from Theorem~\ref{thm:threereal}. Define $g:M\to\dom(f):(X,u,v,w,z)\mapsto (\varphi(X^\top u),v,w,z)$, $I=(-\rho,\rho)\cap (-1,1)$, $I'=(-\rho',\rho')\cap(-1,1)$, $a=\sup_{x\in \bar{I}}|\varphi(x)|$, $a'=\sup_{x\in \bar{I'}}|\psi(x)|$, and
\begin{align*}
    A = \intr(\varphi(I))^n\times (I'/a)^m\times(-1/m,1/m)^m\times (-\rho''/a',\rho''/a')^\ell.
\end{align*}
Note that $A$ is nonempty because $\varphi$ is nontrivial at zero. Furthermore, $A$ has positive Lebesgue measure since it is both nonempty and open.
Let $(u'',v,w,z)\in A$. Then there exists $u'\in I^n$ such that $\varphi(u')=u''$. Set $X=[u'\quad]^\top$ and $u=e_1$. Then $g(X,u,v,w,z)=(u'',v,w,z)$. So, $A\subset \img(g)$. Thus, $\img(g)\not\subset \var(f)$ since $\var(f)$ is Lebesgue measure zero. So, the result holds with $f\circ g$.
\end{proof}

Now, we extend to when the final hidden layer matrix is not necessarily rank-one.

\begin{theorem}
\label{thm:narrowrankfull}
Let $n,d,m,\ell\in\N$.
Let $\varphi:\R\to\R$, $\psi:\R\to\R$, and $\phi:\R\to\R$ be real analytic at zero with radii of convergence $\rho$, $\rho'$, and $\rho''$ respectively. Assume $\varphi$ is nontrivial at zero. Assume $\psi$ and $\phi$ are not polynomials at zero.
Define
\begin{align*}
    M=\{(X,u,v,W)\in\R^{d\times n}\times\R^d\times\R^m\times\R^{m\times \ell}\mid |&x_i^\top u|<\rho,\quad |\varphi(x_i^\top u)v_j|<\rho',\\
    &|\psi(\varphi(x_i^\top u)v^\top)w_k|<\rho''~\forall (i,j,k)\}.
\end{align*}
Then $0\in M$, $M$ is open, and there exists a nontrivial real analytic function $f:M\to\R$ such that, for all $(X,u,v,W)\in\var(f)^\compl$,
\begin{align*}
    \rank\left(\psi\left(\varphi\left(X^\top u\right)v^\top \right)\sly \phi\left(\psi\left(\varphi\left(X^\top u\right)v^\top \right)W \right)\right)\ge \min\{\ell,\lfloor n/(m-1)\rfloor\}(m-1).
\end{align*}
\end{theorem}
\begin{proof}
Let $f:M\to \R$ be the sum of squares of minors of order $\min\{\ell,\lfloor n/(m-1)\rfloor\}(m-1)$. To see that $f$ is nontrivial, let $(X,u,v,w,z)\in \var(g)^\compl$, where $g$ is the nontrivial real analytic function from Theorem~\ref{thm:narrowrank}, and set $W=wz^\top$.
\end{proof}

We can prove a result about four layer FNNs with first layer width equal to one by applying Lemma~\ref{lma:everything} and Theorem~\ref{thm:narrowrankfull}, but, with one more lemma, we can actually prove a result for general FNNs.

\begin{lemma}
\label{lma:reduce}
Let $L\in\N$ such that $L\ge 3$.
Let $\psi_\ell:\R\to\R$ for each $\ell\in[L]$.
Let $d\in\N$.
Set $m_0=d$. Let $m_1,\ldots,m_L\in\N$.
Let $u_\ell\in\R^{m_\ell}~\forall \ell\in[L-1]$.
Define $W_\ell=[u_\ell\quad]^\top\in\R^{m_{\ell-1}\times m_\ell}~\forall \ell\in[L-1]$.
Let $c_\ell$ be the first entry of $u_\ell$ for each $\ell\in[L-2]$.
Let $W_L\in\R^{m_{L-1}\times m_L}$.
Let $v\in\R^{m_L}$.
Let $X\in\R^{d\times n}$.
Then
\begin{align*}
    F(X,W_1,\ldots,W_L,v) = F(X,c_1e_1^\top,c_2,\ldots,c_{L-2},u_{L-1},W_L,v).
\end{align*}
\end{lemma}
\begin{proof}
First, $\psi_1(W_1^\top X)=\psi_1(u_1e_1^\top X)$. Second, $\psi_2(W_2^\top\psi_1(W_1^\top X)) = \psi_2(u_2\psi_1(c_1e_1^\top X))$. Third, let $\ell\in\{2,\ldots,L-2\}$ and suppose
\begin{align*}
    \psi_\ell\left(W_\ell^\top\cdots\psi_1\left(W_1^\top X\right)\cdots\right)=\psi_\ell\left(u_\ell\psi_{\ell-1}\left(c_{\ell-1}\cdots\psi_1\left(c_1e_1^\top X\right)\cdots\right)\right).
\end{align*}
Then,
\begin{align*}
    \psi_{\ell+1}\left(W_{\ell+1}^\top\cdots\psi_1\left(W_1^\top X\right)\cdots\right)&=\psi_{\ell+1}\left(W_{\ell+1}^\top\psi_\ell\left(u_\ell\psi_{\ell-1}\left(c_{\ell-1}\cdots\psi_1\left(c_1e_1^\top X\right)\cdots\right)\right)\right)\\
    &=\psi_{\ell+1}\left(u_{\ell+1}\psi_\ell\left(c_\ell\cdots\psi_1\left(c_1e_1^\top X\right)\cdots\right)\right).
\end{align*}
So, by induction,
\begin{align*}
    \psi_{L-1}\left(W_{L-1}^\top\cdots\psi_1\left(W_1^\top X\right)\cdots\right)=\psi_{L-1}\left(u_{L-1}\psi_{L-2}\left(c_{L-2}\cdots\psi_1\left(c_1e_1^\top X\right)\cdots\right)\right),
\end{align*}
proving the result.
\end{proof}

Lemma~\ref{lma:reduce} shows how to reduce a general FNN to a FNN with four layers and first layer width equal to one. Now, we are ready to prove our final result.

\begin{theorem}
\label{thm:deep}
Let $L\in\N$ such that $L\ge 3$.
Let $\psi_\ell:\R\to\R$ be real analytic at a point and nontrivial there for each $\ell\in[L-2]$.
Let $\psi_\ell:\R\to\R$ be real analytic at a point and not a polynomial there for each $\ell\in\{L-1,L\}$.
Let $d\in\N$.
Set $m_0=d$. Let $m_1,\ldots,m_L\in\N$.
Let $n\in\N$. Assume $m_L\ge 2\lceil n/(m_{L-1}-1)\rceil$.
Then there exists a nontrivial real analytic function $f:\R^{d\times n}\backslash\{0\}\to\R$ such that, for all $X\in\var(f)^\compl$ and $y\in\R^n$, there exists $W_\ell\in\R^{m_{\ell-1}\times m_\ell}~\forall\ell\in [L]$, $b_\ell\in\R^{m_\ell}~\forall \ell\in[L]$, and $v\in\R^{m_L}$ such that
\begin{align*}
    y^\top=v^\top\psi_L\left(W_L^\top\cdots \psi_1\left(W_1^\top X+b_1\one_n^\top \right)\cdots+b_L\one_n^\top \right).
\end{align*}
\end{theorem}
\begin{proof}
Since the only requirement on $n,(m_\ell)$ is that they satisfy $m_L\ge 2\lceil n/(m_{L-1}-1)\rceil$, we can assume, without loss of generality, that $(m_{L-1}-1)|n$ by including additional generic data.
Set $m_0'=m_0$, $m_\ell'=1~\forall\ell\in[L-2]$, $m_{L-1}'=m_{L-1}$, and $m_L'=\lfloor m_L/2\rfloor$. For each $\ell\in[L-2]$, let $\eta_\ell\in\R$ be a point where $\psi_\ell$ is real analytic and nontrivial. For each $\ell\in\{L-1,L\}$, let $\eta_\ell\in\R$ be a point where $\psi_\ell$ is real analytic and not a polynomial.
By setting $b_\ell=\eta_\ell\one_{m_\ell}~\forall \ell\in[L]$, we can assume, without loss of generality, that $\eta_\ell=0~\forall\ell\in[L]$ and remove the bias vectors. Set $v'=\one_{m_L'}$. Let $W_\ell'\in\R\backslash\{0\}~\forall\ell\in\{2,\ldots,L-2\}$ and define $\varphi=\psi_{L-2}(W_{L-2}'^\top\cdots \psi_2(W_2'^\top\psi_1(u^\top\cdot))\cdots )$. Applying Theorem~\ref{thm:narrowrankfull}, there exists a nontrivial real analytic function $g:M\to\R$ such that, for all $(X,u,z,W)\in\var(g)^\compl$, $\rank(G_{(m_\ell'),n}(X,u,W_2,\ldots,W_{L-2},z^\top,W)=n$. Let $(X',u',z',W')\in\var(g)^\compl$. Using similar steps as in the proof of Theorem~\ref{thm:threelayer}, we can define a nontrivial real analytic function $f:\R^{d\times n}\backslash\{0\}\to\R$ such that, for all $X\in\var(f)^\compl$, $F_{(m_0',\ldots,m_{L-1}',m_L),n}(X,\cdot)$ is surjective by Lemma~\ref{lma:everything}. But, for all $X\in\R^{d\times n}$, $\img (F_{(m_0',\ldots,m_{L-1}',m_L),n}(X,\cdot))\subset \img (F_{(m_\ell),n}(X,\cdot))$ by Lemma~\ref{lma:reduce}. Thus, for all $X\in\var(f)^\compl$, $F_{(m_\ell),n}(X,\cdot)$ is surjective, completing the proof.
\end{proof}

Theorem~\ref{thm:deep} shows that an $L$-layer FNN can interpolate $\Omega(m_{L-1}m_L)$ generic data points, but, in principle, it should be able to interpolate $\Theta(\sum_{\ell=1}^Lm_\ell m_{\ell-1})$ generic data points. These are of the same order when $L=3$ or when the number of neurons is being minimized, as we will show in the next section. But, more generally, to precisely determine the interpolation power of a deep FNN we would have to lower bound the generic rank of the full Jacobian rather than just the Jacobian of the final hidden layer. We leave this as a future research direction. 

\section{Necessary and sufficient number of neurons}
\label{sec:sufficient}

By Theorem~\ref{thm:necessary},
\begin{align*}
    \sqrt{2nd'+(d\lor d'+1)^2-2d\land d'-4L+5}-d\lor d'+d'+L-2
\end{align*}
neurons are necessary for an $(L+1)$-layer FNN to interpolate $n$ generic points in $\R^d\times\R^{d'}$. By Theorem~\ref{thm:deep}, $m_L\ge 2\lceil n/(m_{L-1}-1)\rceil$ is sufficient for an $(L+1)$-layer FNN to interpolate $n$ generic points in $\R^d\times\R$. But the sufficient condition actually leads to the following condition on the number of neurons sufficient to interpolate $n$ generic points in $\R^d\times\R^{d'}$.

\begin{theorem}
\label{thm:sufficient}
Let $n,d,d',L\in\N$ with $L\ge 2$.
Let $\psi_\ell:\R\to\R$ be real analytic at a point and nontrivial there for each $\ell\in[L-2]$.
Let $\psi_\ell:\R\to\R$ be real analytic at a point and not a polynomial there for each $\ell\in\{L-1,L\}$. Then there is a sequence of widths $(m_\ell)$ with less than
\begin{align*}
    2\sqrt{2nd'}+d'+L
\end{align*}
neurons such that an $(L+1)$-layer FNN with activations $(\psi_\ell)$ and widths $(m_\ell)$ can interpolate $n$ generic points in $\R^d\times\R^{d'}$.
\end{theorem}
\begin{proof}
Define $m_\ell=~\forall \ell\in[L-2]$. Define $m_{L-1}=\lceil\sqrt{2nd'}\rceil+1$ and $m_L=\lceil \sqrt{2n/d'}\rceil$. Then $m_L\ge 2\lceil n/(m_{L-1}-1)\rceil$ so we can apply Theorem~\ref{thm:threelayer} or Theorem~\ref{thm:deep} to get that an $(L+1)$-layer FNN with activations $(\psi_\ell)$ and widths $(m_\ell)$ can interpolate $n$ generic points in $\R^d\times\R$. But note that
\begin{align*}
    \begin{bmatrix}
        v_1&&\\&\ddots&\\&&v_{d'}
    \end{bmatrix}^\top A=\begin{bmatrix}
        v_1^\top A&&\\&\ddots&\\&&v_{d'}^\top A
    \end{bmatrix}
\end{align*}
for any matrix $A$. Thus, an $(L+1)$-layer FNN with activations $(\psi_\ell)$ and widths $(m_1,\ldots,m_{L-1},d'm_L)$ can interpolate $n$ generic points in $\R^d\times\R^{d'}$. To complete the proof, note that the number of neurons in $(m_1,\ldots,m_{L-1},d'm_L)$ is $L-2+\lceil\sqrt{2nd'}\rceil+1+\lceil \sqrt{2n/d'}\rceil d'<2\sqrt{2nd'}+d'+L$.
\end{proof}

To compare the necessary and sufficient conditions, assume $d,d'=o(n)$ and $L=o(\sqrt{n})$. Then the necessary number of neurons is $\sqrt{2nd'}+\Omega(1)$ and the sufficient number of neurons is $2\sqrt{2nd'}+\Omega(1)$.

\section{Conclusion}

We showed that for feedforward neural networks with at least three layers mapping from $\R^d$ to $\R^{d'}$, $\sqrt{2nd'}+\Omega(1)$ neurons are necessary to interpolate $n$ generic data points and $2\sqrt{2nd'}+\Omega(1)$ neurons are sufficient. The most technical part of the proof was showing that the Jacobian with respect to the final hidden layer has close to full generic rank. From there, we applied the Constant Rank Theorem to prove the existence of an interpolating solution. While the final hidden layer has the largest share of parameters in a three layer network, this is not necessarily the case for a deep network. Thus, it is a future research direction to construct the interpolating solution with respect to the full Jacobian and so prove an optimal sufficient condition on the number of parameters needed for interpolation.

\bibliographystyle{main}
\bibliography{main}

\end{document}